\documentclass[11pt]{article}
\usepackage{graphicx} % Required for inserting images

\title{Multigroup Fairness and Omniprediction: \\ Separations and Equivalences}

\author{S\'ilvia Casacuberta\thanks{Supported by a Knight-Hennessy Scholarship, the Simons Foundation Collaboration on the Theory of Algorithmic Fairness, and the Simons Foundation investigators award 17351.}\\
Stanford University
\and
Parikshit Gopalan\\
Apple
\and Varun Kanade\\
University of Oxford
\and
Omer Reingold\thanks{Supported by the Simons Foundation Collaboration on the Theory of Algorithmic Fairness, and the Simons Foundation investigators award 17351.}\\
Stanford University
\and Konstantinos Stavropoulos\thanks{This work was done in part while the author was an intern at Apple and while at The University of Texas at Austin. The author was supported in part by the NSF AI Institute for Foundations of Machine Learning (IFML) and, during the 2026–27 academic year, by the Institute for Advanced Study and an anonymous donor.}\\
Institute for Advanced Study
\and  
Pranay Tankala\thanks{Work done during an internship at Apple.}\\
Harvard University}

\usepackage[utf8]{inputenc}
\usepackage{amsmath, amssymb, amsthm, bbm, graphicx, url}
\usepackage{amsfonts,fullpage}
\usepackage{float}
\usepackage{graphicx}

\usepackage{bbm}
\usepackage{xcolor}
\newif\ifshownotes
\shownotesfalse   % hide notes
\usepackage{natbib}
\usepackage[boxruled,linesnumbered,vlined]{algorithm2e}
\usepackage[pagebackref]{hyperref}
\hypersetup{
    colorlinks=true,
    linktocpage=true,
    linkcolor=blue!50!black,  % Darker blue for internal links
    citecolor=blue!50!black,  % Darker blue for citations
    urlcolor=blue!50!black    % Darker blue for external links
}
\usepackage[capitalize]{cleveref}
\usepackage{dsfont}
\usepackage{xspace}
\usepackage{comment}
\usepackage{tikz}

\newcommand{\eat}[1]{}

\usepackage{thmtools}
\usepackage{thm-restate}

\newtheorem*{definition*}{Definition}
\newtheorem*{proposition*}{Proposition}
\newtheorem*{corollary*}{Corollary}

\newtheorem{theorem}{Theorem}[section]
\newtheorem{lemma}[theorem]{Lemma}
\newtheorem{definition}[theorem]{Definition}

\newtheorem{corollary}[theorem]{Corollary}

\newtheorem{remark}[theorem]{Remark}

\newcommand{\wh}{\widehat}

\newcommand{\R}{\mathbb{R}}
\newcommand{\X}{\mathcal{X}}

\newcommand{\mA}{\mathcal{A}}
\newcommand{\mC}{\mathcal{C}}
\newcommand{\mD}{\mathcal{D}}

\newcommand{\mL}{\mathcal{L}}
\newcommand{\mW}{\mathcal{W}}
\newcommand{\mX}{\mathcal{X}}
\newcommand{\mP}{\mathcal{P}}

\newcommand{\tp}{p}
\newcommand{\ps}{\ensuremath{p^*}}
\newcommand{\ty}{\mathbf{y}_p}

\newcommand{\x}{\mathbf{x}}

\newcommand{\y}{\mathbf{y}}

\newcommand{\zo}{\ensuremath{\{0,1\}}}

\renewcommand{\hat}{\wh}
\newcommand{\eps}{\varepsilon}

\newcommand{\abs}[1]{\ensuremath \Bigl\lvert #1 \Bigr\rvert}

\newcommand{\ind}[1]{\ensuremath{\mathbf{1}[#1]}}

\newcommand{\C}{\mathcal{C}}

\DeclareMathOperator*{\E}{\mathbf{E}}

\DeclareMathOperator*{\argmin}{arg\,min}

\newcommand{\Ber}{\mathrm{Ber}}

\newcommand{\OPLC}{\mathrm{OP}_{\mL, \mC}}
\newcommand{\OPLCzero}{\mathrm{OP}_{\mL, \mC, 0}}
\newcommand{\OPLCeps}{\mathrm{OP}_{\mL, \mC, \epsilon}}

\newcommand{\btw}{\mathrm{Btwn}}
\newcommand{\conv}{\mathrm{Conv}}
\newcommand{\lpr}{\ell^{\mathrm{pr}}}

\newcommand{\ECE}{\mathsf{ECE}}

\newcommand{\EAE}{\mathsf{EAE}}

\usepackage{enumitem}

\definecolor{pastelgreen}{RGB}{179,205,182}
\definecolor{pastelblue}{RGB}{177,197,222}
\definecolor{pastelred}{RGB}{226,182,181}

\definecolor{darkgreen}{RGB}{76,140,92}
\definecolor{darkblue}{RGB}{64,104,170}
\definecolor{darkred}{RGB}{170,70,70}

\usepackage{bm}

\renewcommand{\epsilon}{\varepsilon}

\usepackage[T1]{fontenc}
\usepackage[utf8]{inputenc}
\begin{document}

\maketitle
Omniprediction \cite{omni} is a learning guarantee which requires a single predictor to be competitive relative to the best hypothesis from a benchmark class for any loss chosen from a family of loss functions. Loss Outcome Indistinguishability (loss OI for short) \cite{lossOI} is a stronger notion that implies omniprediction. It requires the predicted distribution on labels to be indistinguishable from the true distribution to tests that depend on the loss functions and the benchmark class. 
Multiaccuracy and multicalibration \cite{hkrr2018} are multigroup fairness notions that generalize classical notions of calibration and accuracy in expectation.  Most known learning algorithms for omniprediction (both for the standard notion and for strengthenings like loss OI) rely on some version of these multigroup fairness notions, or on an intermediate notion called calibrated multiaccuracy~\cite{lossOI}. We ask if this is necessary:
\begin{center}
{\em Does omniprediction require some form of multigroup fairness?}
\end{center}

We show that the answer is no for (plain) omniprediction, and yes for loss OI.  
\begin{itemize} \item \textbf{Omniprediction does not require accuracy in expectation.}
A sequence of works shows that multicalibration or calibrated multiaccuracy imply omniprediction \citep{omni, lossOI, OKK25}. We rule out even a weak converse, by showing that omniprediction for proper losses does not imply even accuracy in expectation, a much weaker notion than any of calibration, multiaccuracy, or multicalibration. 

\item \textbf{Loss OI requires calibrated multiaccuracy.} 
The work of \cite{lossOI}  showed how to achieve loss OI from a combination of calibration and multiaccuracy. We show a converse: loss OI is equivalent to a form of calibrated multiaccuracy. 
\end{itemize}

\thispagestyle{empty}
\newpage

\setcounter{page}{1}

\section{Introduction}

Consider the setting of binary classification, where labeled data points $(\x, \y)$ are sampled from an unknown distribution over a domain $\X \times \{0, 1\}$, and our goal is to learn a \emph{predictor}, which is a function $\tp : \X \to [0, 1]$ that estimates the probability that the label $\y$ will be $1$ given features $\x \in \X$. In this work, we investigate the relationship between the notion of \emph{omniprediction} introduced by \cite{omni} and its variants \cite{lossOI, gopalan2024swap}, and the various \emph{multigroup fairness} criteria satisfied by the predictor $\tp$.

\emph{Omniprediction} is a powerful new paradigm for loss minimization, first proposed by~\cite{omni}.
Specifically, an \emph{$(\mL, \mC)$-omnipredictor} is a predictor that performs as well as any hypothesis from a baseline class $\mC$, for any loss function coming from a family $\mL$.
The first construction for such a ``loss-oblivious'' learning notion was achieved through a surprising connection to \emph{multigroup fairness} notions; specifically, through the notion of $\C$-\emph{multicalibration}, which first originated in the algorithmic fairness literature \cite{hkrr2018}.

\emph{$\mC$-multicalibration} requires the predictor $\tp$ to be \emph{calibrated} conditioned on the event $\x \in S$, for all $S$ in a prespecified collection $\mC \subseteq 2^{\X}$ of protected subsets of the domain. The weaker notion of \emph{$\mC$-multiaccuracy} from \cite{hkrr2018} requires that the expectation of the predictor matches the expectation of the label, conditioned on membership in $S \in \mC$. An intermediate notion is \emph{$\mC$-calibrated multiaccuracy}, introduced in \cite{lossOI}, which requires both global calibration (without conditioning) and $\mC$-multiaccuracy. These notions fall under the umbrella of \emph{multigroup fairness}, a viewpoint that attempts to bridge the gap between \emph{group fairness} and \emph{individual fairness} \cite{dwork2012fairness} by allowing the collection of protected groups $\mC$ in the domain to be numerous and overlapping. %For example, $\mC$-multicalibration ensures that being assigned a score of $p(\x) \in [0, 1]$ ``means the same thing'' to an evaluator regardless of one's group membership(s).

Although originally motivated by fairness considerations in prediction, multicalibration  and multiaccuracy can be viewed as extensions of two classical notions in forecasting: \emph{calibration} ($\E[\y|\tp(\x)] = p(\x)$, \cite{Dawid}) and \emph{accuracy in expectation} ($\E[\y] = \E[\tp(\x)]$), respectively. $\C$-multiaccuracy is also equivalent to the notion of $\C$-\emph{computational indistinguishability} in the pseudorandomness and complexity theory literature, where constructions of multiaccurate predictors originally appeared~\citep{TrevisanTV09}. An active body of research has drawn connections between these notions and topics in learning theory and theoretical computer science, including strong guarantees for learnability \cite{omni, lossOI, OKK25}, pseudorandomness \cite{OI, DworkLLT} and computational complexity \cite{casacuberta2024complexity, CGKR26}. 

Multiaccuracy and multicalibration are motivated by fairness, and their definitions do not reference loss functions. 
But \cite{omni} showed that
these multigroup fairness notions imply surprisingly strong and robust loss minimization guarantees:
namely, $\mC$-multicalibration (appropriately generalized to real-valued function classes $\mC$) yields omniprediction for all convex, Lipschitz loss functions.\footnote{This requires extending the notion of multicalibration to allow for real-valued functions $c$ rather than indicators of subsets.} 
This initiated an exciting line of research, which derived various flavors of omniprediction from different multigroup fairness notions \citep{lossOI, gopalan2024swap, gopalan2024omnipredictors, GargJRR24, HuTY25, DworkHIPT25, OKK25, CGKR26}. Even more recent state-of-the-art constructions of omnipredictors go through some form of multigroup fairness: either multicalibration \citep{omni, GargJRR24} or calibrated multiaccuracy \citep{DworkHIPT25, OKK25, HuTY25}. 
This naturally raises the following fundamental question:

\begin{center}
{\em Is (some form of) multigroup fairness necessary for omniprediction? \\ What about calibration or accuracy in expectation?}
\end{center}

These were considered in the work of \cite{gopalan2024swap}, who call it a glaring question in the theory of omniprediction. They showed that multicalibration is both necessary and sufficient for a stronger version of omniprediction called \emph{swap omniprediction}.\footnote{Technically, they show an equivalence with \emph{swap multicalibration}, but this is essentially the same as standard multicalibration up to polynomial changes in the error parameters.} However, the question of whether multigroup fairness is required for the standard notion of omniprediction raised in that work has remained open.

A strengthening of omniprediction is the notion of \emph{$(\mL, \mC)$-loss outcome indistinguishability (OI)}, introduced by \cite{lossOI}. Intuitively, given a loss function $\ell \in \mL$ and hypothesis $c \in \mC$, a predictor $p$ satisfies loss OI if the excess loss that $p$ suffers compared to the hypothesis $c$ is roughly unchanged on replacing the real outcomes $\y$ with \emph{modeled} or \emph{simulated} outcomes $\y_\tp|\x \sim \Ber(\tp(\x))$. The name ``outcome indistinguishability'' is derived from the more general OI framework introduced by \cite{OI}. It was shown by \cite{lossOI} that calibrated multiaccuracy implies loss OI, which in turn implies omniprediction. 

The motivation for studying loss OI is that it can be tailored to specific classes of loss functions, in contrast with known results on omniprediction (e.g., \cite{omni} which applies to all Lipschitz, convex losses), and this can give more efficient algorithms (see the discussion in \cite{lossOI}). Indeed, the best known algorithms for omniprediction for specific loss families \cite{OKK25,HuTY25} use the loss OI formulation as their starting point. For instance, the weaker condition of \emph{proper} calibrated multiaccuracy is crucially used by \cite{OKK25} to achieve loss OI for proper losses with optimal rate.
However, as in the case of omniprediction, it was unclear from previous work whether either calibration or multiaccuracy is \emph{necessary} for loss OI.  The work of \cite{gopalan2024swap} showed an equivalence between a stronger variant called \emph{swap} loss OI, and a suitably defined notion of multicalibration.

\subsection{Results Overview}

In this work, we show a sharp contrast between the notions of omniprediction and loss OI: while omniprediction does \emph{not} require any of calibration, multiaccuracy, or multicalibration, loss OI is equivalent to a form of calibrated multiaccuracy.

\paragraph{Omniprediction does not require accuracy in expectation (\Cref{sec:closure}).} 
Formally, given a hypothesis class $\mC$, a family $\mL$ of loss functions, and $\epsilon > 0$, an $(\mL, \mC, \epsilon)$-\emph{omnipredictor} is a predictor $p$ such that for every $\ell \in \mL$,
\[
    \E[\ell(\y, k_{\ell} \circ p(\x))] \leq \min_{c \in \C} \E[\ell(\y, c(\x))] + \epsilon,
\]
where $k_{\ell}(v) = \argmin_{t \in \mA} \ell(v, t)$ is the optimal action function for each loss function $\ell: [0,1] \times \mA \rightarrow \mathbb{R}$.
Recall that a predictor is \emph{accurate in expectation} if $\E[\y] = \E[p(\x)]$.
The \emph{expected accuracy error} ($\EAE$) of $p$ is naturally defined as $\EAE(p) = \lvert\E[\y]-\E[p(\x)]\rvert$.

While most known constructions of omnipredictors require multigroup fairness notions \citep{omni, GargJRR24,DworkHIPT25, OKK25, HuTY25}, we show that the converse is not true; namely, we show that omniprediction for proper losses does not require accuracy in expectation. We require that the class $\mC$ and marginal distribution $\mD_\X$ over $\X$ are such
that there exists a Boolean function $f$ which is strongly hard on average for $\mC$, meaning that for every $c \in \mC$,
\begin{align} 
\label{eq:hard-on-av}
    \left| \E_{\mD_\X}\left[c(\x)\Big(f(\x)-\dfrac{1}{2}\Big)\right] \right| \leq \epsilon.
\end{align}

\begin{theorem}\label{thm:sec3-main-informal}
    Fix $\epsilon > 0$ and a distribution $\mD_{\X}$ over $\X$. Let $\C$ be any hypothesis class that admits a strong hard-on-average function $f$ with respect to $\mD_{\X}$, as in Equation \eqref{eq:hard-on-av}. Consider the distribution $\mD = (\x, f(\x))$ for $\x \sim \mD_\X$. 
    There exists a predictor $p: \X \rightarrow [0,1]$ such that (1) $p$ is an $(\mL, \mC, O(\epsilon))$-omnipredictor for the class $\mL$ of all convex, Lipschitz losses, and (2) $\EAE(p) \geq 1/4$.
\end{theorem}

As we observe in Remark~\ref{remark:proper}, the result can also be extended to proper losses using the recent result of \cite{OKK25}.
Since accuracy in expectation is an especially mild condition, implied by any one of calibration, multiaccuracy, or multicalibration, this result implies that \emph{none} of these conditions are required for omniprediction. This also shows that neither calibration nor multiaccuracy is necessary for agnostic learning, although together they are sufficient, as shown in \cite{CGKR26}. The work of \cite{gopalan2024swap} asked if there is {\em a characterization of standard omniprediction in terms of multigroup fairness notions}; our work answers this in the negative.

The theorem statement above does not address the efficiency of computing the predictor $p$, but it is essentially identical to that of computing the hard function $f$. If $f$ is efficiently computable by some class (it would need to be more expressive than $\mC$ itself), then so is $p$. 

The key conceptual idea in the proof is to study the closure properties of the set of omnipredictors. We show that they are closed under a between-ness operation, which implies that the set of omnipredictors is quite large. We then show that such a large set must contain points that significantly violate even simple conditions like calibration and accuracy in expectation. 
The proof involves three steps.
\begin{enumerate}
    \item First we prove that if $q$ is an omnipredictor, then any predictor $r$ that lies between $q$ and the Bayes optimal predictor $p^* = \E[\y|x]$, point-wise for every $x \in \X$, is also an omnipredictor (\cref{thm:between}).
    Intuitively, moving a predictor point-wise towards the Bayes optimal $p^*$ can only further reduce its loss, and so the omniprediction guarantee satisfied by $q$ is preserved.
    \item Second, we apply this observation starting from an omnipredictor $q$ satisfying $\E[|p^*(\x)-q(\x)|] \geq \alpha$ for some large $\alpha$. By decomposing the $\ell_1$ loss between $p^*$ and $q$ into the regions $\{x : p^*(x) \geq q(x)\}$ and $\{x : q(x) > p^*(x)\}$, we construct a predictor $r$ that lies between $q$ and $p^*$, and hence is also an omnipredictor by~\cref{thm:between}, but which has $\EAE(r) \geq \alpha/2$ (\cref{cor:ae}).
    We remark that this holds even in the case where $q$ (and hence $r$) are perfect omnipredictors (i.e., with $\epsilon=0$).
    \item Third, we show that for any hypothesis class $\C$ that admits a strong hard-on-average Boolean function, there exists an omnipredictor $q$ satisfying $\E[|p^*(\x)-q(\x)|]=1/2$, and hence by Lemma~\ref{cor:ae} an omnipredictor $r$ satisfying $\EAE(r) \geq 1/4$ (\cref{thm:eae_large}).
\end{enumerate}

\paragraph{Loss OI requires calibrated multiaccuracy (Section~\ref{sec:loss-oi}).} 
We show that the situation is very different for the stronger notion of loss OI (which implies omniprediction \cite{lossOI}). 
Indeed, we prove that, in contrast to omniprediction, both calibration and multiaccuracy are necessary as long as $\mL$ contains some simple and natural loss functions. 
This result establishes a satisfying converse to a key result of \cite{lossOI}, which showed that calibration and multiaccuracy are together sufficient for loss OI. 

Formally, given a predictor $p$, we associate a simulated label distribution $\mD(\tp)$ on $(\x, \ty) \in \X \times \zo$, where $\x \sim \mD_\X$ and $\ty|\x \sim \Ber(p(\x))$.
For $\ell \in \mL$ and $c \in \mC$, we define the distinguisher $u_{\ell, c}: \X \times  \zo \times [0,1]$ as $ u_{\ell, c}(x, y, v) =  \ell(y, c(x)) - \ell(y, k_\ell(v))$.
We can now formally state the definitions of loss OI, decision OI, and hypothesis OI \cite{lossOI}:

\begin{definition}
    We say that the predictor $\tp:\X \to [0,1]$ satisfies:
    \begin{itemize}
        \item $(\mL, \mC, \alpha)$-\emph{loss OI} if for every $\ell \in \mL$ and $c \in \mC$,
        \[
            \abs{\E\bigl[u_{\ell, c}(\x, \y, \tp(\x))\bigr] - \E\bigl[u_{\ell, c}(\x, \ty, \tp(\x))\bigr]} \leq \alpha.
        \]
        \item $(\mL, \beta)$-\emph{decision OI} if for every $\ell \in \mL$, 
        \[
            \abs{\E\big[\ell(\y, k_\ell(\tp(\x)))\big] - \E\big[\ell(\ty, k_\ell(\tp(\x)))\big]} \leq \beta.
        \]
        \item $(\mL, \mC, \gamma)$-\emph{hypothesis OI} if for every $\ell \in \mL$ and every $c \in \mC$, 
        \[
            \abs{\E\big[\ell(\y, c(\x))\big] - \E\big[\ell(\ty, c(\x))\big]} \leq \gamma. 
        \]
    \end{itemize}
\end{definition}

The work of \cite{lossOI} shows that decision OI corresponds to a form of calibration, whereas
hypothesis OI corresponds to a form of multiaccuracy.
They also show that decision and hypothesis OI together (i.e., a form of calibrated multiaccuracy) suffice for loss OI.

In this paper we show a converse to this statement. 
The only mild condition that we require on the loss function is for it to have a \emph{neutral action}: an action $t \in \mA$ where the loss is the same whether $y$ is $0$ or $1$. Formally, if $\partial \ell(t) = \ell(1, t) - \ell(0, t)$, then a neutral action is one where
$\partial \ell(t) = 0$. For instance, $t=1/2$ is a neutral action for the squared loss and the cross-entropy loss. In general, a simple shift-based transformation can transform an arbitrary loss into one with a neutral action, while essentially preserving its behavior (see \Cref{sec:loss-oi}).

\begin{theorem}
    Given $\ell$ with a neutral action and $\mA \subseteq \mC$,\footnote{We think of each $t \in \mA$ as the constant function that always takes the value $t$.} if a predictor satisfies $(\ell, \C, \alpha)$-loss OI, then it also satisfies both $(\ell, \alpha)$-decision OI and $(\ell, \mC, 2\alpha)$-hypothesis OI.
\end{theorem}

The idea behind our proof is the following.
The key property ensured by a neutral action $t$ is that the loss is independent of the (distribution of the) label and consequently $\E[\ell(\y, t)] = \E[\ell(\y_p, t)]$.
Under the assumption that $p$ satisfies $(\ell, \C, \alpha)$-loss OI and that $t \in \mC$, this implies that the inequality $\big|\E\bigl[u_{\ell, c}(\x, \y, \tp(\x))\bigr] - \E\bigl[u_{\ell, c}(\x, \ty, \tp(\x))\bigr]\big| \leq \alpha$ is in particular satisfied with  the distinguisher $u_{\ell, t}$ (i.e., for the neutral action $t$).
Using the fact that $\E[\ell(\y, t)] = \E[\ell(\y_p, t)]$, we can turn the loss OI guarantee for $u_{\ell, t}$ into a decision OI guarantee. From here we derive the hypothesis OI guarantee for every $c \in \C$ by applying a triangle inequality and also using the general loss OI guarantee. As we shall discuss in \Cref{sec:loss-oi}, the assumption of a neutral action is not strictly necessary and can easily be relaxed or replaced with other, similarly mild assumptions.

\paragraph{Summary.} Our results demonstrate an important distinction between the notions of omniprediction and loss OI: the former does not enforce any form of multigroup fairness guarantees (or even accuracy in expectation), while the latter requires a form of calibrated multiaccuracy. This resolves an open question in the literature \cite{gopalan2024swap}. 
Our findings suggest that although most algorithmic results on omniprediction go through the ``loss OI route'', they might be sacrificing efficiency, since we now know that omnipredictors are ``more plentiful'' than
predictors satisfying loss OI.

\subsection{Related Work}

Our work continues a line of investigation into connections between multigroup fairness and various notions of learning. Much of this work aims to derive learning guarantees from multigroup fairness notions, including calibration \citep{omni, lossOI, gollakota2023agnostically, DworkLLT, gopalan2024swap, gopalan2024omnipredictors, GargJRR24, HuTY25, DworkHIPT25, OKK25, CGKR26}. Papers which investigate the reverse connection include \cite{gopalan2024swap, CGKR26}. Alternative algorithms for omniprediction, which are based on 2-player game methods rather than boosting algorithms, do achieve optimal error bounds for omniprediction (up to polynomial factors) through online-to-batch conversions \cite{gibbs2025sample, balakrishnan2025panprediction}.

The work of \cite{gopalan2024swap} presents tight connections between the swap analogues of omniprediction and loss OI on the one hand to swap multicalibration for families $\mC'$ derived from $\mL$ and $\mC$ on the other hand. From this equivalence they conclude that standard omniprediction, which they show is weaker than swap omniprediction, is unlikely to require multicalibration, which they show is equivalent to swap multicalibration, using a rather indirect chain of reasoning. However, this does not rule out the necessity of some weaker multigroup notion, or even calibration.

Multigroup fairness notions have several connections to complexity theory and pseudorandomness. An indistinguishability perspective on these notions was presented in \cite{OI}, which motivated the definition of loss OI. Connections to hardcore set construction in complexity theory and questions in pseudorandomness were first noted by \cite{TrevisanTV09}, and expanded upon in \cite{casacuberta2024complexity, CGKR26, HuVadhan25}. \cite{DworkLLT} connect multigroup fairness notions of varying strengths to graph regularity lemmas in combinatorics; see also \cite{dwork2025efficient} for more connections to complexity theory. 
In recent work, Gibbs and Tibshirani \cite{gibbs2025sample} proved that when calibration error is measured via expected calibration error ($\ECE$), both multicalibration and calibrated multiaccuracy  are strictly harder problems than omniprediction or loss OI in terms of sample complexity.  Their result demonstrates a tension between the desiderata of low $\ECE$ and multiaccuracy. Their result relies critically on $\ECE$ being the calibration measure: while they frame their result as showing that {\em omniprediction (or rather loss OI) is easier to achieve than multicalibration or calibrated multiaccuracy}, this is only true if calibration error is measured using $\ECE$.  Indeed, the result of \cite{OKK25} shows how one can obtain loss OI with optimal sample complexity using predictors that satisfy multiaccuracy and a weaker form of calibration they call \emph{proper calibration}, which captures decision OI for proper losses. Our result shows that proper calibration and multiaccuracy are both necessary to obtain loss OI for proper losses. 
\section{Preliminaries}
\label{sec:preliminaries}

We will use boldface (e.g. $\x$) to denote random variables. Given a distribution $\mD$ over pairs $(\x, \y) \in \X \times \{0, 1\}$, we denote the marginal distribution of $\x$ by $\mD_\X$.
A \emph{predictor} is a function $p : \mX \to [0, 1]$, a \emph{hypothesis} is a function $c : \mX \to [-1, 1]$, and a \emph{loss} is a function $\ell : \{0, 1\} \times \mA \to \R$ for a set of \emph{actions} $\mA$. The \emph{Bayes optimal predictor} is given by $\ps(x) = \E[\y|x]$. 
The reason to have both predictors and hypotheses is that certain notions like calibration are traditionally defined for $[0,1]$-valued functions, whereas others like correlation are more natural for $[-1,1]$-valued functions. To every predictor $\tp$, we associate a simulated label distribution $\mD(\tp)$ on $(\x, \ty) \in \X \times \zo$ where $\x \sim \mD_\X$ and $\ty|\x \sim \Ber(p(\x))$. In other words, $\mD(\tp)$ is a distribution whose Bayes optimal predictor is $\tp$. By this notation, $\mD = \mD(p^*)$; we prefer the former for simplicity.

\paragraph{Multigroup fairness notions.}
We now define \emph{multiaccuracy}, \emph{multicalibration}, and (global, weight-restricted) \emph{calibration}:

\begin{definition}[\cite{hkrr2018,kearns2018,GopalanKSZ22}]
    Fix a hypothesis class $\mC$ of hypotheses $c:\X \to [-1,1]$ and a class $\mW$ of \emph{weight functions} $w : [0, 1] \to [-1, 1]$. We say a predictor $p$ satisfies \emph{$(\mC, \delta)$-multiaccuracy} (with respect to a distribution $\mD$ over $(\x, \y)$ pairs) if
    \[
        \max_{c \in \mC} \, \abs{\E\bigl[c(\x)(\y-\tp(\x))\bigr]} \le \delta.
    \]
    We say that $\tp$ satisfies \emph{$(\mC, \delta)$-multicalibration} if
    \[
        \max_{c \in \mC} \, \E\abs{\E\bigl[c(\x)(\y-\tp(\x)) \mid \tp(\x)\bigr]} \le \delta.
    \]
    We say that $\tp$ satisfies
    \emph{$(\mW, \delta)$-weight restricted calibration} if
    \[
        \max_{w \in \mW} \, \abs{\E\bigl[w(\tp(\x))(\y-\tp(\x))\bigr]} \le \delta.
    \]
\end{definition}
Assume that $\mC$ and $\mW$ are closed under complement and contain the constant function $\mathrm{1}$. 
We define the 
{\em expected accuracy error} ($\EAE$) and the \emph{expected calibration error} ($\ECE$) 
of a predictor as:
\[
    \EAE(\tp) = \abs{\E[\y] - \E[\tp(\x)]},
    \quad 
    \ECE(\tp) = \E\abs{\E[\y|\tp(\x)] - \tp(\x)}.
\]
Observe that $\EAE(p) \leq \ECE(p)$, $(\mC,\delta)$-multicalibration implies $\ECE(p) \leq \delta$, and either one of $(\mC,\delta)$-multiaccuracy or $(\mW, \delta)$-weight restricted calibration implies $\EAE(p) \leq \delta$ (since $\mathrm{1} \in \mC, \mW$). 

\paragraph{Loss functions.} In a loss minimization task, on input $\x \sim \mD_\X$, we take an action $t \in \mA$ (typically based on a prediction $p(\x)$ of the label), and suffer loss $\ell(\y, t)$ for a loss function $\ell: \zo \times \mA \to \R$. We allow for abstract action spaces $\mA$ which might be discrete or continuous. Typically, we pick the optimal action from a hypothesis class of functions $\mC = \{c: \X \to \mA\}$. This results in the loss minimization problem 
\begin{align}
\label{eq:agnostic-ell}
\min_{c \in \mC} \E_{(\x, \y) \sim \mD} [\ell(\y, c(\x))], 
\end{align}
which we shall refer to as \emph{agnostically learning $\mC$ with the loss function $\ell$}.

We extend $\ell$ linearly to a function $\ell: [0,1] \times \mA \to \R$ by defining
\[ \ell(p, t) = p\ell(1, t) + (1- p)\ell(0, t) = \E_{\y \sim \Ber(p)}[\ell(\y, t)].\]
An important class of loss functions is the class of \emph{proper losses}:
\begin{definition}
    We say that a loss $\ell$ is \emph{proper} if $\mA = [0,1]$ and $\ell(p,q) \geq \ell(p,p)$ for all $q \in [0,1]$.
\end{definition}

In other words, if $\y \sim \Ber(p)$, then the expected loss is minimized by the action $q = p$. 

Recall that given a loss function $\ell: [0,1] \times \mA \to \R$, the optimal action function $k_\ell: [0,1] \to \mA$ is given by $k_\ell(v) = \argmin_{t \in \mA} \ell(v,t).$
\cite{KLST23} reduce the case of arbitrary loss functions to proper losses, by showing how we can post-process an arbitrary loss function in order to make it proper:

\begin{lemma}[\cite{KLST23}]
    \label{lem:prop}
    Given a loss function $\ell: [0,1] \times \mA \to \R$, the loss function $\lpr:[0,1] \times [0,1] \to \R$ defined as $\lpr(p, q) = \ell(p, k_\ell(q))$ is a proper loss function. 
\end{lemma}

We will use the well-known characterization of proper loss functions in terms of Bregman divergences \citep{GR07, Nielsen2}, which we define below. 

\begin{definition}
    \label{def:prop}
    Given a concave function $h: [0,1] \to \R$, the \emph{associated Bregman divergence} $D_h: [0,1] \times [0,1] \to \R^+$ is given by 
    \begin{align}
        D_h(p,q) = h(q) + (p -q) h'(q) - h(p), 
    \end{align}
    where $h'$ is the derivative of $h$ in the case that $h$ happens to be differentiable, or an arbitrary superderivative of $h$ otherwise.
\end{definition}

\begin{theorem}[\cite{GR07, Nielsen2}]\label{thm:proper-loss}
    If $\lpr: [0,1] \times [0,1] \to \R$ is a proper loss, then there exists a concave function $h$ such that
    \[
        \lpr(p, q) = D_h(p,q) + h(p) = h(q) + (p-q)h'(q),
    \]
    where $h'$ is the derivative of $h$ in the case that $h$ happens to be differentiable, or a superderivative of $h$ otherwise. Moreover, $h$ is given by the formula $h(p) = \lpr(p,p)$, and $h'(p) = \lpr(1,p) - \lpr(0,p)$ is a superderivative.
\end{theorem}

The key property of $\lpr(p,q)$ is that it is linear in $p$, so it can be viewed as the linear extension of
$\lpr(y,q) = h(q) + (y-q)h'(q)$.
If the loss functions we consider are bounded to lie in the range $[0,1]$, then it follows that
$h'(q) = \lpr(1,q) - \lpr(0,q) \in [-1,1]$.

\section{Omniprediction and Accuracy in Expectation}
\label{sec:closure}

In this section, we state and prove our result showing that omniprediction does not even require the very mild notion of accuracy in expectation.

Informally, a predictor $\tp$ is an \emph{$(\mL, \mC)$-omnipredictor} if for every loss $\ell \in \mL$, we would achieve better loss by not taking actions prescribed by any $c \in \mC$ but instead taking the \emph{optimal action according to $\tp$}, which we now define.
Specifically, given a loss function $\ell: [0,1] \times \mA \to \R$, the {\em optimal action} function $k_\ell: [0,1] \to \mA$ is defined as $k_\ell(v) = \argmin_{t \in \mA} \ell(v,t)$.\footnote{If there are multiple actions that minimize $\ell$, we pick one arbitrarily.}

We remark that the function $k_\ell$ is loss-specific, but is typically simple for standard loss functions.  An omnipredictor allows for simultaneous loss minimization across a large class of loss functions, without considering each loss separately:

\begin{definition}[\cite{omni}]
For a family $\mL$ of loss functions, a hypothesis class $\mC$ and $\epsilon > 0$, an $(\mL, \mC, \epsilon)$-\emph{omnipredictor} is a predictor $\tp$ such that for every $\ell \in \mL$, we have
\begin{align}
    \label{eq:op}
    \E[\ell(\y, k_\ell \circ \tp(\x))] \leq \min_{c \in \mC} \E[\ell(\y, c(\x))] + \epsilon.
\end{align}
\end{definition} 

The right hand side is the loss incurred by agnostically learning $\mC$ with loss $\ell$ up to an additive factor of $\epsilon$, the optimal solution to which will typically depend on the choice of $\ell$. The power of omniprediction comes from the fact that the predictor $\tp$ is learned from the data in one shot. 
Its predictions can be used to minimize any loss from $\mL$ relative to $\mC$ without any further learning.
The existence of omnipredictors follows from observing that $p^*$ is an omnipredictor for all $(\mL, \mC)$. The key challenge with omnipredictors is efficient learning: prior work requires the algorithm to output an efficiently computable predictor.
\cite{omni} first showed how to construct omnipredictors using a weak agnostic learning oracle for $\C$.

Let $\mP = \{p: \X \rightarrow [0,1]\}$ denote the space of all predictors, and let
$\OPLCeps \subseteq \mP$
denote the set of all $(\mL, \mC, \epsilon)$-omnipredictors. In this section, we explore closure properties of $\OPLCeps$ under two simple operations. We can also consider the case of \emph{perfect} omniprediction, when $\epsilon = 0$; we will denote $\OPLCzero$ by $\OPLC$.

\paragraph{Convex combinations.} Given $p, q \in \mP$ and $\lambda \in [0,1]$, define the predictor $r = \lambda p + (1 - \lambda) q \in \mP$ by $r(x) = \lambda p(x) + (1 - \lambda) q(x)$. Let $\conv(p, q)$ be the set of all such convex combinations of $p, q$. 

\paragraph{Between-ness.} Given $p, q \in \mP$, we define $\btw(p,q)$ to be the set of predictors $r \in \mP$ such that for every $x \in \X$, $\min(p(x), q(x)) \leq r(x) \leq \max(p(x), q(x))$.

We say that such an $r$ \emph{lies between} $p$ and $q$. Alternately, we can view each $r \in \btw(p,q)$ as being obtained by choosing $\lambda(x) \in [0,1]$ independently for every $x \in \X$, and  setting $r(x) = \lambda(x)p(x) + (1 - \lambda(x))q(x)$. 
Note that the set $\btw(p,q)$ is convex and $\conv(p,q) \subseteq \btw(p,q)$. Observe that $\btw(p, q)$ may be $|\mX|$-dimensional, e.g. if $p(x) \neq q(x)$ for every $x \in \mX$, as opposed to $\conv(p,q)$ which is one-dimensional, so $\btw(p, q)$ is a much larger set. 

We first show that if $q$ is an omnipredictor, then any $r$ which lies between $q$ and the Bayes optimal predictor $p^*$ is also an omnipredictor. The main idea is to show that starting from an omnipredictor and moving towards the Bayes optimal predictor only reduces loss, hence it preserves the property of omniprediction. The proof uses a result of \cite{KLST23} to reduce to the case of proper losses, and a characterization of proper losses in terms of Bregman divergences of concave functions~\citep{GR07, Nielsen2}. Formally:

\begin{restatable}{theorem}{between}
\label{thm:between}
    For any classes $\mL$ and $\mC$, if $q \in \OPLCeps$, then $\btw(p^*, q) \subseteq \OPLCeps$.
\end{restatable}

\begin{proof}
    Let $\ell \in \mL$ be a loss function, and let $q \in \OPLCeps$, so that Equation \eqref{eq:op} holds. Let $r \in \btw(p^*, q)$. Our goal is to show that 
    \begin{align} 
    \label{eq:ell-redux0}
        \E_{(\x, \y) \sim \mD}[\ell(\y, k_\ell(r(\x))) ] \leq  \E_{(\x, \y) \sim \mD}[\ell(\y, k_\ell(q(\x)))].
    \end{align}
    Since $q \in \OPLCeps$, this will imply that $r \in \OPLCeps$, which proves the result.

    By Lemma \ref{lem:prop}, there exists a proper loss $\lpr$ so that Equation \eqref{eq:ell-redux0} is equivalent to 
    \begin{align} 
        \label{eq:ell-redux2}
        \E_{(\x, \y) \sim \mD}[\lpr(\y, r(\x)) ] \leq  \E_{(\x, \y) \sim \mD}[\lpr(\y, q(\x))].
    \end{align}
    We will in fact show that Equation \eqref{eq:ell-redux2} holds even conditioned on $\x$:
    \begin{align} 
        \label{eq:lpr-redux}
        \E_{\y|\x}[\lpr(\y, r(\x))] = \lpr(p^*(\x), r(\x)) \leq \lpr(p^*(\x), q(\x)) = \E_{\y|\x}[\lpr(\y, q(\x))] .
    \end{align}
    
    Consider $\x$ such that $q(\x) \leq p^*(\x)$. The case $q(\x) \geq p^*(\x)$ is handled similarly. Since the inequality holds pointwise for every $x$, we will drop the $x$ and show that for $q \leq r \leq p^*$, with all values in $[0,1]$, 
    \begin{align}
    \label{eq:simple}
        \lpr(p^*,q) \geq \lpr(p^*,r).
    \end{align}
    
    By Definition \ref{def:prop}, there exists a concave function $h$ (depending on $\lpr$) so that Equation \eqref{eq:simple} is equivalent to  
    \begin{align}
    \label{eq:sg-inc}
        h(q) + (p^* - q) h'(q) \geq h(r) + (p^* - r)h'(r).         
    \end{align}
    Geometrically, this  says that the supergradient approximation to $h(p^*)$ decreases (and hence improves) as we move from $q$ towards $p^*$.  
    To prove Equation~\eqref{eq:sg-inc}, we note that  
    \begin{align*}
        h(q) + (p^* -q)h'(q) & = h(q) + (r - q)h'(q) + (p^* - r)h'(q)\\
                             & \geq h(r) + (p^* -r)h'(r),
    \end{align*}
where we use $h(q) + (r -q)h'(q) \geq h(r)$ by concavity, and $h'(q) \geq h'(r)$ since $q < r$ and concavity implies $h'(q) \geq h'(r)$ for the chosen superderivatives.
\end{proof}

An important consequence is that omnipredictors do not even need to be accurate in expectation. Note that accuracy in expectation is a very weak requirement which is implied by both calibration and multiaccuracy. This is even in the case of perfect omniprediction with $\epsilon = 0$.

\begin{restatable}{lemma}{ae}
    \label{cor:ae}
    Assume there exists $q \in \OPLC$ such that $\E[|p^*(\x) - q(\x)|] \geq \alpha$. Then there exists $r \in \OPLC$ such that $\EAE(r) \geq \alpha/2$. 
\end{restatable}

\begin{proof}
    Note that
    \begin{align}
        \E[|p^*(\x) - q(\x)|] &= \E[\ind{p^*(\x) \geq q(\x)}(p^*(\x) - q(\x))] + \E[\ind{q(\x) > p^*(\x)} (q(\x) - p^*(\x))] \label{eqn:absdecompose}.
    \end{align}
    Thus, at least one of the two terms on the RHS of Eq.~\eqref{eqn:absdecompose} is at least $\alpha/2$; without loss of generality, assume that $\E[\ind{q(\x) > p^*(\x)}(q(\x) - p^*(\x))] \geq \alpha/2$. 
    
    Let us define the predictor $r$ as 
    \begin{align} 
        r(x) = \begin{cases} p^*(x) \ \text{if} \ q(x) \leq p^*(x)\\
        q(x) \ \ \text{if} \ q(x) > p^*(x).
        \end{cases}
    \end{align}
    It follows that $r \in \btw(p^*, q)$, and so by Theorem~\ref{thm:between}, $r \in \btw(p^*, q) \subseteq \OPLC$.
    But we have
    \begin{align*}
        \EAE(r) = \abs{\E[r(\x) - \y]} = \abs{\E[r(\x) - p^*(\x)]}
                = \abs{\E[\ind{q(\x) > p^*(\x)}(q(\x) - p^*(\x))]} \geq \alpha/2.
    \end{align*}
\end{proof}

For simple classes, we can explicitly construct examples of a $p^*$ and an omnipredictor for which accuracy in expectation does not hold. We outline such a construction for the case of coordinate functions when the ground truth $p^*$ is the parity function. 

\begin{itemize}[leftmargin=1em, itemsep=0em]
    \item $\mD$ is the distribution where $\x$ is uniform on $\zo^k$ and $\y = \chi(\x)$ is the parity of the bits. 
    \item Let $\mL$ be all convex, Lipschitz losses and $\mC = \{x \mapsto x_i \mid i \in [k]\}$ be all coordinate functions. (We could alternatively consider parities, conjunctions, or disjunctions on strictly fewer than $k$ bits.)
    \item The constant predictor $q(x) =1/2$ is $(\mC, 0)$-multicalibrated. So it is an $(\mL, \mC, 0)$-omnipredictor by the main result of \cite{omni}.
    \item The predictor $r$ defined as $r(x) = 1/2$ when $\chi(x) = 0$ and $r(x) = 1$ when $\chi(x) =1$ is also an $(\mL, \mC, 0)$-omnipredictor since it lies between $q$ and $p^*$. But $\E[r(\x)] = 3/4$ whereas $\E[\y] =1/2$, so $\EAE(r) = 1/4$.  
\end{itemize}

The question then is whether there always exist omnipredictors $q$ for which $\E[|p^*(\x) - q(\x)|]$ is large. Under mild complexity-theoretic assumptions on the hypothesis class $\mC$ and the marginal distribution $\mD_\mX$, we show that such omnipredictors always exist. 
Our assumption asserts the existence of a 
Boolean function which is extremely hard for $\mC$ under the distribution $\mD_\X$. 

\begin{definition}\label{def:hard-on-avg}
    For a distribution $\mD_\X$ on $\X$, a hypothesis class $\mC$ and $\eps > 0$, we say that the Boolean function $f : \mX \rightarrow \{0, 1\}$ is $(1/2 - \eps)$-\emph{hard} for $\mC$ under $\mD_\X$ if for every $c \in \mC$,
\begin{equation}\label{eq:PRF-cond-C}
    \left| \E_{\mD_\X}\left[c(\x)\Big(f(\x)-\dfrac{1}{2}\Big)\right] \right| \leq \epsilon.
\end{equation}
\end{definition}

To see why this implies average-case hardness, note that for Boolean functions $c$, and a balanced function $f$, this definition implies that 
\[ \Pr_{\mD_\X}[c(x) = f(x)] \leq 1/2 + 2\eps.\]

Our result applies to the class $\mL_{c}$ of convex, Lipschitz loss functions. These are loss functions where $\mA = \R$, hence $\ell:\zo \times \R \to \R$, and  $\ell(y,t)$ is a convex, Lipschitz function of $t$ for $y \in \zo$. The main result of \cite{omni} shows that if $p$ is $\C$-multicalibrated, then it is an $(\mL_{c}, \mC)$-omnipredictor. We use this to show the following: 

\begin{restatable}{theorem}{eaelarge}\label{thm:eae_large}
    Let $\mC$ be a hypothesis class of functions with range $[-1, 1]$, $\mD_\X$ be a distribution over $\mX$, and let $\mL$ be the class of all convex, Lipschitz losses.\footnote{By \emph{convex}, we mean that both functions $\ell(0, \cdot)$ and $\ell(1, \cdot)$ are convex, for each $\ell \in \mL$. As in \cite{omni} we can relax this to include all $(B, \epsilon)$-nice functions and drop the global Lipschitz requirement.} Suppose that there exists a Boolean function $p^* : \mX \rightarrow \{0, 1\}$ and $\epsilon > 0$ such that $p^*$ is $(1/2 - \eps)$-hard for $\mC$ under $\mD_\X$.
Then, for the distribution $\mD = \mD(p^*)$, there exists a predictor $r : \mX \rightarrow [0, 1]$ such that $r$ is an $(\mL, \mC, O(\epsilon))$-omnipredictor and $\EAE(r) \geq \frac{1}{4}$. 
\end{restatable}

\begin{proof}
    Consider the function $q(x) = 1/2$. For $\mD = \mD(p^*)$, it is straightforward to see that $q$ is $(\mC, \epsilon)$-multicalibrated. Since $p^*$ is $(1/2-\epsilon)$-hard for $\mC$ under $\mD_\X$, for each $c \in \mC$, 
    \[
        |\E[c(\x)(\y-q(\x))]| = \left| \E\left[c(\x)\Big(p^*(\x)-\dfrac{1}{2}\Big)\right] \right| \leq \epsilon.  
    \]
     Then the main result of \citep{omni} implies that $q$ is an $(\mL, \mC, O(\epsilon))$-omnipredictor where $\mL$ is the class of convex, Lipschitz losses; here the hidden constant in the $O$-notation depends only on the Lipschitz constant of the class of loss functions.

    Since the range of $p^*$ is $\{0, 1\}$,
    \[ \E\left[|q(\x) - p^*(\x)|\right] = \frac{1}{2}, \]
    and the conclusion follows by applying Lemma~\ref{cor:ae}.
\end{proof}

\begin{remark}\label{remark:proper}
    The result of Theorem~\ref{thm:eae_large} extends beyond the class $\mL_{c}$ of convex, Lipschitz loss functions. 
    The recent work by \cite{OKK25} shows that if $q$ is $(\mathrm{Th}(\mC), \epsilon)$-multicalibrated, then it is an $(\mL_{\textrm{BV}}, \C, O(\epsilon))$-omnipredictor for the class $\mL_{\textrm{BV}}$ of bounded variation losses, which include all convex, all Lipschitz, and all
    proper losses, where $\mathrm{Th}(\C)$ denotes the concept class consisting of thresholds applied to $\C$.
    Applying this result to the proof of Theorem~\ref{thm:eae_large}, we similarly obtain a predictor $r$ that is an $(\mL_{\textrm{BV}}, \mathrm{Th}(\mC), O(\epsilon))$-omnipredictor and $\EAE(r) \geq \frac{1}{4}$. 
\end{remark}

Another consequence of this construction is that neither calibration nor multiaccuracy is  necessary for strong agnostic learning (although together they are sufficient \citep{lossOI, CGKR26}). The proof of the result below follows directly from the above construction by considering the $\ell_1$ loss.
This result can be made more general as in the case of \cref{thm:eae_large}. 

\begin{corollary}
    Under the assumptions of \cref{thm:eae_large}, there exists a predictor $p: \X \to [0,1]$ and a post-processing function $\kappa: [0,1] \rightarrow [-1, 1]$ such that
    \[
        \E[(2\y-1)\kappa(p(\x))] \geq \max_{c \in \C} \E[(2\y-1)c(\x)] - O(\epsilon),
    \] 
    but $\EAE(p) \geq 1/4$. 
\end{corollary}

For the $\ell_1$ loss, the optimal post-processing function corresponds to $\kappa(v) = 2 \cdot \mathbbm{1}[v \geq 1/2]-1$.

\paragraph{On the existence of average-case hard functions.}
    For what classes $\mC$ and distributions $\mD_\X$ can we expect there to be a Boolean function $f$ that satisfies Definition~\ref{def:hard-on-avg}?
    Clearly we need the $\mD_\X$ distribution to have sufficient min-entropy: else even a constant function will give a sufficiently good approximation to $f$. So let us assume this to be the case. We will consider the standard complexity theory setting where $\X = \zo^n$ and $\mD_\X$ is the uniform distribution; analogous statements hold for other domains and (high-entropy) distributions, but are tedious to formalize. 
    
    We now focus on the class $\mC$. 
    The answer depends on whether we want the omnipredictor to be efficiently computable or not. The definition of an omnipredictor itself does not require it to be efficiently computable. Yet, as noted in \cite{omni}, the definition is most interesting when the predictor is not only efficiently computable but also learnable; this typically means that it lies in a more powerful class than $\mC$. Hence, we will consider both settings: the information-theoretic setting where the omnipredictor is not required to be computationally efficient, and the setting where it is. 

    Consider the information-theoretic setting where we do not require $p^*$ or $p$ to be efficiently computable. As long as functions in the class $\mC$ have bounded description lengths, or small covering numbers (in the continuous setting), a simple counting argument implies that there exists a $p^*$ which is average-case hard for $\mC$.  This gives an omnipredictor $p$ satisfying the guarantees of Theorem \ref{thm:eae_large}. 

    The more challenging setting is where we require a Boolean $p^*$ that is hard for $\mC$ under $\mD_\X$, but is efficiently computable by a more powerful class. This requires making unproven but plausible computational assumptions about $\mC$. Similar assumptions are used in the seminal work of Nisan and Wigderson \cite{nisan1994hardness} to construct pseudorandom generators that can be used to fool statistical tests modeled by functions from $\mC$. Such assumptions are stronger than worst-case hardness assumptions, but weaker than the average-case hardness assumptions on which cryptography is based; see Impagliazzo’s well-known Five Worlds survey \cite{Impagliazzo95}.

\section{Loss OI and Calibrated Multiaccuracy}
\label{sec:loss-oi}

In this section, we present our result characterizing loss OI \citep{lossOI} as equivalent in power to calibrated multiaccuracy.
Loss OI is an instantiation of the general notion of sample-access \emph{outcome indistinguishability} (OI) from \cite{OI}, which concerns the ability of a prespecified class of \emph{distinguisher} functions to tell apart genuine data points of the form $(\x, \y, \tp(\x))$ from modeled ones of the form $(\x, \ty, \tp(\x))$, where $\ty|\x \sim \Ber(\tp(\x))$. It was shown in \cite{lossOI} that loss OI, while stronger than omniprediction, is nevertheless implied by calibrated multiaccuracy. The present result can be viewed as a converse to that prior result, fully characterizing loss OI as \emph{equivalent} to calibrated multiaccuracy.
This is in contrast to our result in the previous section, which showed that the weaker notion of omniprediction does not even require accuracy in expectation. We start with the definition of loss outcome indistinguishability:

\begin{definition}[Loss Outcome Indistinguishability \citep{lossOI}]
For $\ell \in \mL$ and $c \in \mC$, we define the distinguisher $u_{\ell, c}: \X \times  \zo \times [0,1]$ as $u_{\ell, c}(x, y, v) =  \ell(y, c(x)) - \ell(y, k_\ell(v)).$
We say the predictor $\tp:\X \to [0,1]$ is $(\mL, \mC, \alpha)$-\emph{loss OI} if for every $\ell \in \mL$ and $c \in \mC$,
\begin{align}
    \abs{\E\bigl[u_{\ell, c}(\x, \y, \tp(\x))\bigr] - \E\bigl[u_{\ell, c}(\x, \ty, \tp(\x))\bigr]} \leq \alpha.
\end{align}

\end{definition}

To gain intuition for the distinguisher $u_{\ell, c}$ and the notion of loss OI, we make three key observations. First, we note that the test $u_{\ell, c}$ compares---for each labeled instance $(x, y)$ and prediction value $v$---the difference in loss suffered by two actions:
the action $c(x)$ suggested by the hypothesis $c$ and the action $k_\ell(v)$, which is the optimal action when $\y \sim \Ber(v)$. 
Next, we note $u_{\ell, c}$ is always nonnegative in expectation when fed prediction values $v$ generated by the predictor $p$ and simulated labels $\y_p|\x\sim\Ber(\tp(\x))$. To see this, observe that by the definition of $k_\ell$,
\[ \E\bigl[u_{\ell, c}(\x, \y_\tp, \tp(\x))\bigr] = \E \bigl[\ell(\y_\tp, c(\x)) - \ell(\y_p, k_\ell(\tp(\x)))\bigr] \geq 0.\]
This is because $\y_p|\x \sim \Ber(p(\x))$, so $k_\ell \circ p(\x)$ incurs the lowest loss among all possible actions. Finally, we observe that if a predictor $p$ satisfies (perfect) loss OI, then it too incurs less loss than $c$, because we can switch $\y_p$ and $\y$ to get
\begin{align}
\label{eq:lossOI-op}
    \E_{\mD}\bigl[\ell(\y, c(\x))\bigr] - \E_{\mD}\bigl[\ell(\y, k_\ell(\tp(\x)))\bigr] =  \E\bigl[u_{\ell, c}(\x, \y, \tp(\x))\bigr] = \E\bigl[u_{\ell, c}(\x, \ty, \tp(\x))\bigr] \geq 0.
\end{align}
This is precisely the condition required for omniprediction.

Having established the relationship between loss OI and omniprediction, we now recall the two simpler notions of \emph{decision OI} and \emph{hypothesis OI}, also introduced in \cite{lossOI}, which together are sufficient for loss OI.

\begin{definition}[Decision OI]
    The predictor $\tp:\X \to [0,1]$ satisfies $(\mL, \beta)$-\emph{decision OI} if for every $\ell \in \mL$ it holds that
\begin{align}
    \abs{\E\big[\ell(\y, k_\ell(\tp(\x)))\big] - \E\big[\ell(\ty, k_\ell(\tp(\x)))\big]} \leq \beta.
\end{align}
\end{definition}

\begin{definition}[Hypothesis OI]
    The predictor $\tp:\X \to [0,1]$ satisfies $(\mL, \mC, \gamma)$-\emph{hypothesis OI} if for every $\ell \in \mL$ and every $c \in \mC$, it holds that
\begin{align}
    \abs{\E\big[\ell(\y, c(\x))\big] - \E\big[\ell(\ty, c(\x))\big]} \leq \gamma. 
\end{align}
\end{definition}

They prove a decomposition lemma showing that decision OI and hypothesis OI together imply loss OI:

\begin{lemma}[Lemma 4.4 in \cite{lossOI}]
\label{thm:loss-oi-prior}
    If a predictor satisfies $(\mL, \beta)$-decision OI and $(\mL, \mC, \gamma)$-hypothesis OI, then it satisfies $(\mL, \mC, \beta + \gamma)$-loss OI. 
\end{lemma}

In the same paper, they show that decision OI is a form of weight-restricted calibration (as defined in \cite{GopalanKSZ22}), and hypothesis OI is multiaccuracy for a class of hypotheses derived from $\mL$ and $\mC$. To state these prior results, let $\partial \ell(t) = \ell(1, t) - \ell(0, t)$ denote the discrete partial derivative of $\ell$ with respect to the label.

\begin{lemma}[Characterizing decision and hypothesis OI \cite{lossOI}]
\label{thm:decision-oi}\;
\begin{itemize}
\item  A predictor $p$ satisfies $(\ell,\beta)$-decision OI if and only if $p$ satisfies weight-restricted calibration with error at most $\beta$ with respect to the weight function $\partial \ell \circ k_\ell$.
\item A predictor $p$ satisfies $(\ell, \mC, \gamma)$-hypothesis OI if and only if $p$ satisfies multiaccuracy with error at most $\gamma$ with respect to the function class $\partial \ell \circ \mC$.
\end{itemize}
\end{lemma}

The main result of this section is essentially a converse to \Cref{thm:loss-oi-prior}. We show that under a very mild condition on the loss function $\ell$ and class $\mC$, any predictor satisfying $(\ell, \mC, \alpha)$-loss OI must \emph{also} satisfy both $(\ell, O(\alpha))$-decision OI and $(\ell, \mC, O(\alpha))$-hypothesis OI. Combined with \cref{thm:decision-oi}, which says that decision OI is a form of calibration and hypothesis OI is a form of multiaccuracy, this establishes the desired equivalence between loss OI and calibrated multiaccuracy. We prove our equivalence for loss functions $\ell$ that have a \emph{neutral} action:

\begin{definition}[Neutrality]
    We say that an action $t \in \mA$ is \emph{neutral} for a loss function $\ell : \{0, 1\} \times \mA \to \R$ if $\partial \ell(t) = 0$.
\end{definition}
Many natural loss functions have neutral actions. For example, $t = 1/2$ is a neutral action for the squared loss and cross-entropy loss, and $t = 0$ is neutral for the logistic loss. Although there do exist loss functions without neutral actions, such losses can always be easily tweaked to correct this anomaly. For example, the loss function $\ell(y, p) = y$ is proper, but the ranges of $\ell(1, p)$ and $\ell(0, p)$ are disjoint, so there cannot exist a neutral action. In this example, the additive dependence on $y$ was the cause of the problem. The following lemma shows that the presence of such a term is the only possible obstruction to having a neutral action, and that this obstruction can be bypassed by subtracting away the term:

\begin{restatable}{lemma}{shiftneutral}
\label{thm:shift-neutral}
    Fix a loss function $\hat{\ell} : \{0, 1\} \times \mA \to \R$ and an action $t_0 \in \mA$. Then $t_0$ is neutral for the loss function $\ell(y, t) = \hat{\ell}(y, t) - \partial \hat{\ell}(t_0)y.$ Moreover, if $\mA = [0, 1]$ and $\hat{\ell}$ is proper, then so is $\ell$. 
\end{restatable}

\begin{proof}
    Observe that
    \begin{align*} 
        \ell(y, t) = \hat{\ell}(y, t) - \partial \hat{\ell}(t_0)y  = \hat{\ell}(0,t) + y(\partial \hat{\ell}(t) - \partial \hat{\ell}(t_0)).
    \end{align*}
    Hence $\ell(y,t_0) = \hat{\ell}(0,t_0)$ and $t_0$ is a neutral action for $\ell$. 
    
    Next, recall that $\hat{\ell}$ is proper if and only if there exists a concave function $h$ and a superderivative $h'$ of $h$ such that $\hat{\ell}(y, p) = h(p) + h'(p)(y-p)$.
    If $h(p)$ is concave, then so is the function $h(p) + c \cdot p$ for any constant $c$, of which $h'(p) + c$ is a valid superderivative. Thus, the following loss function is also proper:
    \[
        \bigl(h(p) + c \cdot p\bigr) + \bigl(h'(p) + c\bigr)(y-p).
    \]
    When $c = -\partial\hat{\ell}(t_0)$, the expression above exactly coincides with the function $\ell$ of interest.
\end{proof}

The last mild assumption that we need concerns the class $\mC$, not the loss $\ell$. We require $\mC$ to contain all the {\em constant} action functions $c(x) = t$ for each fixed action $t \in \mA$. We denote this by $\mA \subseteq \mC$. With this notation in hand, we are ready to state the main result of this section:

\begin{restatable}{theorem}{lossoicalma}
\label{thm:loss-oi-cal-ma}
    Suppose the loss function $\ell : \{0, 1\} \times \mA \to \R$ has a neutral action and $\mA \subseteq \mC$. Then, a predictor $p$ satisfies $(\ell, \mC, O(\alpha))$-loss OI if and only if it satisfies both $(\partial \ell \circ k_\ell, O(\alpha))$-weight restricted calibration and $(\partial \ell \circ \mC, O(\alpha))$-multiaccuracy.
\end{restatable}

Once again, we emphasize that the neutral action assumption is a mild one, and is essentially only necessary to rule out the trivial loss function $\ell(y,p)=y$ that does not depend on the prediction $p$, as well as rescalings and small perturbations thereof (i.e. functions where $\partial \ell$ is close to a constant function). We could, for example, replace the neutral action assumption with the similarly mild requirements that $\lvert \mathrm{range}(\partial \ell)\rvert \ge \Omega(1)$ and $\mathrm{range}(\ell) \subseteq [-1, 1]$, which also rule out these bad cases.

\Cref{thm:loss-oi-cal-ma} extends the previously known \Cref{thm:loss-oi-prior}, which was a one-directional implication. Bearing in mind the connections between decision OI and calibration and between hypothesis OI and multiaccuracy (\Cref{thm:decision-oi}), we see that the new direction of \Cref{thm:loss-oi-cal-ma} follows immediately from the following lemma regarding losses with neutral actions.

\begin{restatable}{lemma}{lossoikey}
\label{thm:loss-oi-key-lemma}
    Assume that the loss function $\ell : \{0, 1\} \times \mA \to \R$ has a neutral action and that $\mA \subseteq \mC$. If a predictor $\tp$ satisfies $(\ell, \mC, \alpha)$-loss OI, then it also satisfies both $(\ell, \alpha)$-decision OI and $(\ell, \mC, 2\alpha)$-hypothesis OI.
\end{restatable}

\begin{proof}
Let us consider a neutral action $t$ for $\ell$. Since $t\in \mC$, the loss OI condition on $p$ implies 
    \begin{align}
    \label{eq:oi-t}
        \abs{\E[u_{\ell, t}(\x, \y, \tp(\x))] - \E[u_{\ell, t}(\x, \ty, \tp(\x))]} \leq \alpha.
    \end{align}
    But since $t$ is neutral, the distribution on labels is irrelevant. Hence we have
    \[ \E[\ell(\y, t)] = \E[\ell(\ty, t)] = \ell(0,t) = \ell(1,t).\] 
    Plugging in the definition of $u_{\ell, t}$, namely $ \ell(y, t) - \ell(y, k_\ell(v))$ and using the fact that $\E[\ell(\y, t)] = \E[\ell(\ty, t)]$, the left side of Equation~\eqref{eq:oi-t} equals
    \begin{align}
    & \abs{\E[\ell(\y, k_\ell(\tp(\x))) - \ell(\y, t)]  - \E[\ell(\ty, k_\ell(\tp(\x))) - \ell(\ty, t)]}\notag\\
    & = \abs{\E[\ell(\y, k_\ell(\tp(\x)))] - \E[\ell(\ty, k_\ell(\tp(\x)))]},
    \end{align}
    which implies the claim about decision OI.  For the claim about hypothesis OI, observe that
    \begin{align*}
        \E[\ell(\y, c(\x))] &= \E[u_{\ell, c}(\x, \y, \tp(\x))] + \E[\ell(\y, k_\ell(\tp(\x)))]\\
        \E[\ell(\ty, c(\x))] &= \E[u_{\ell, c}(\x, \y_p, \tp(\x))] + \E[\ell(\ty, k_\ell(\tp(\x)))].
    \end{align*}
    For every $c \in \C$, each expectation on the right side changes only by $\alpha$ in going from $\y$ to $\ty$, by loss OI applied to $u_{\ell, c}$ (for the term $\E[u_{\ell, c}(\x, \y, \tp(\x))]$) and by decision OI applied to $\ell$ (for the term $\E[\ell(\y, k_\ell(\tp(\x)))]$), respectively. Therefore, by the triangle inequality, the expectations on the right side can change by at most $2\alpha$, and so $\abs{\E[\ell(\y, c(\x))] - \E[\ell(\y_p, c(\x))]} \leq 2\alpha$ for every $c \in \C$. 
    This proves the claim about hypothesis OI.
\end{proof}

\paragraph{AI usage.} No AI models were used, either in proving the results or in preparing this manuscript. 

\newpage 

\bibliographystyle{alpha}
\bibliography{refs}

\end{document}